\documentclass[final]{cvpr}

\usepackage{times}
\usepackage{epsfig}
\usepackage{graphicx}
\usepackage{amsmath, bm}
\usepackage{amssymb}

\usepackage{xcolor}
\usepackage{booktabs, multirow}

\usepackage{subfigure}

\usepackage{algorithmic}
\usepackage{algorithm}

\newcommand{\round}[1]{\lfloor #1 \rceil}

\usepackage{amsthm}
\newtheorem{theorem}{Theorem}
\newtheorem{proposition}[theorem]{Proposition}

\usepackage[pagebackref=true,breaklinks=true,colorlinks,bookmarks=false]{hyperref}

\def\cvprPaperID{7525} 
\def\confYear{CVPR 2021}

\begin{document}

\title{iVPF: Numerical Invertible Volume Preserving Flow for \\ Efficient Lossless Compression}

\author{Shifeng Zhang \quad Chen Zhang \quad Ning Kang \quad Zhenguo Li\\
Huawei Noah's Ark Lab\\
{\tt\small \{zhangshifeng4, chenzhang10, kang.ning2, li.zhenguo\}@huawei.com}
}

\maketitle

\begin{abstract}
It is nontrivial to store rapidly growing big data nowadays, which demands high-performance lossless compression techniques.
Likelihood-based generative models have witnessed their success on lossless compression, 
where flow based models are desirable in allowing exact data likelihood optimisation with bijective mappings. 
However, common continuous flows are in contradiction with the discreteness of coding schemes, which requires either 1) imposing strict constraints on flow models that degrades the performance or 2) coding numerous bijective mapping errors which reduces the efficiency.
In this paper, we investigate volume preserving flows for lossless compression and show that a bijective mapping without error is possible. 
We propose Numerical Invertible Volume Preserving Flow (iVPF) which is derived from the general volume preserving flows. 
By introducing novel computation algorithms on flow models, an exact bijective mapping is achieved without any numerical error. 
We also propose a lossless compression algorithm based on iVPF. 
Experiments on various datasets show that the algorithm based on iVPF achieves state-of-the-art compression ratio over lightweight compression algorithms. 
\end{abstract}

\section{Introduction}
Lossless compression is widely applied to efficiently store a growing amount of data in the big data era. 
It is indicated by Shannon's source coding theorem that the optimal codeword length is lower bounded by the entropy of the data distribution~\cite{mackay2003information}. 
Many coding methods~\cite{huffman1952method,witten1987arithmetic,duda2013asymmetric} succeed in reaching the theoretical codelength, and are widely applied in various compression algorithms~\cite{gage1994new,collet2016smaller,rabbani2002jpeg2000,roelofs1999png}.

Probabilistic modelling techniques can be adopted for lossless compression with coding algorithms, in which the minimum possible codelength would be essentially bounded by the negative log-likelihood of the probabilistic model. 
Recent advances of deep generative models have shown great success in maximising the likelihood of various types of data including images~\cite{kingma2013auto,kingma2018glow,salimans2017pixelcnn++,dinh2016density}, videos~\cite{kumar2019videoflow,kalchbrenner2017video} and audios~\cite{prenger2019waveglow,oord2016wavenet}. 
Various lossless compression algorithms have been discovered based on different types of generative models~\cite{mentzer2019practical,townsend2019practical,hoogeboom2019integer,berg2020idf++,ho2019compression,townsend2019hilloc,kingma2019bit,mentzer2020learning,cao2020lossless}. 
They witnessed better compression ratio compared to those without machine learning techniques.

\begin{figure}
    \centering
    \includegraphics[width=0.45\textwidth]{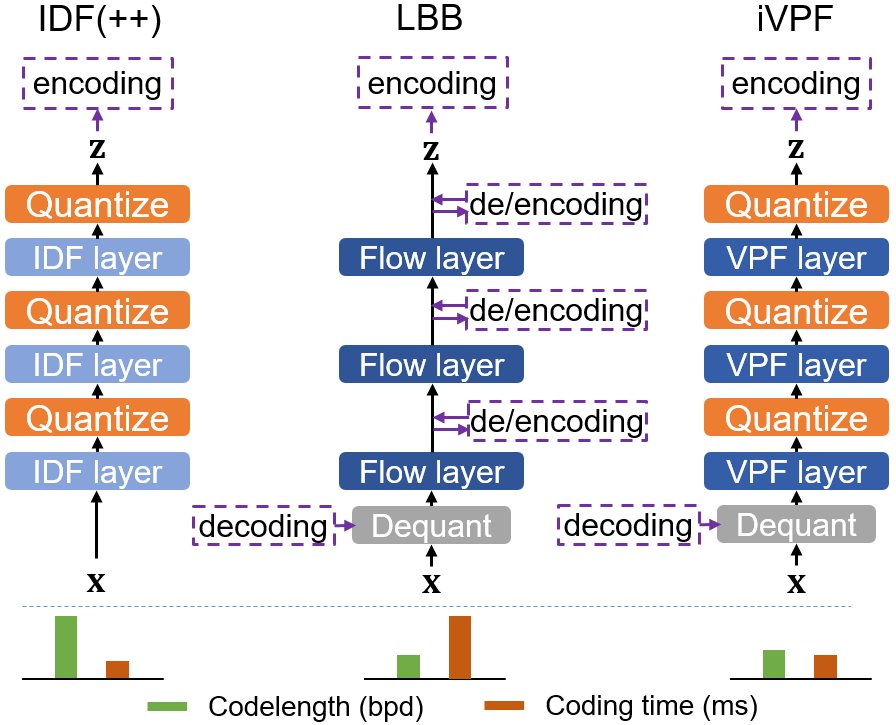}
    \caption{Illustration of IDF(++), LBB and the proposed iVPF model. VPF denotes volume preserving flow. Flow layers with darker colour denotes higher expressive power. The compression codelength and coding time are listed below (lower is better).}
    \label{fig:illustration}
\end{figure}

In this paper, we focus on coding with flow-based generative models~\cite{dinh2016density,kingma2018glow,dinh2016density,prenger2019waveglow,kumar2019videoflow}. 
Although variational auto-encoders (VAEs)~\cite{kingma2013auto,higgins2016beta} and auto-regressive~\cite{salimans2017pixelcnn++,van2016conditional} models are both probabilistic models, they either suffer from inferior likelihood estimation (VAEs), or involve huge computational overhead (auto-regressive). 
Normalising flow models admit bijective mappings between input data and latent variables, and the latter is adopted for likelihood computation given the prior. 
However, common flows are continuous and numerical errors always occur for forward and inverse computation~\cite{behrmann2020understanding}, while lossless compression involves discrete data and any numerical error is unacceptable. 
To solve the above issue, IDF~\cite{hoogeboom2019integer} and IDF++~\cite{berg2020idf++} introduce discrete flows by adding discrete constraints on the model, but the theoretical likelihood is inferior to that of general continuous flows. 
LBB~\cite{ho2019compression} succeeds in coding with general flows and the codelength is close to the theoretical lower bound of negative log-likelihood, but it is extraordinarily time-consuming. 

To circumvent the above issues, we introduce iVPF, a Numerical Invertible Volume Preserving Flow model with high expressivity and efficient coding.
We point out that general flows are almost not bijective in practice, because the real-world data are discrete in nature, resulting the difficulty for lossless compression.
The proposed iVPF is constructed from volume preserving flow~\cite{dinh2014nice, sorrenson2020disentanglement} to achieve the bijective mapping on discrete space, and an efficient lossless compression algorithm based on iVPF is developed. \footnote{Code will be available at \url{https://gitee.com/mindspore/models/tree/master/research/cv/ivpf}.}
The main contributions of this work are summarised as follows:

\begin{enumerate}
\setlength{\itemsep}{0pt}
    \item We investigate the importance of volume preserving flow in the lossless compression context, and propose the novel computation algorithm Modular Affine Transformation (MAT) which achieves exact bijective mapping without any numerical error for coupling layers of volume preserving flows.
    \item Based on MAT, we propose the Numerical Invertible Volume Preserving Flow (iVPF) model and an efficient lossless compression algorithm for iVPF. 
    \item Experiments show the effectiveness of the proposed lossless compression algorithm, where the compression performance is superior over the state-of-the-art lightweight methods.
\end{enumerate}

\section{Related Work}
\noindent
{\bf Entropy Coding for Lossless Compression} We put our work in the regime of machine learning techniques for lossless compression.
As indicated by Shannon's source coding theorem, the best possible lossless compression performance is given by entropy of the data distribution \cite{mackay2003information}.
Various coding schemes, e.g., Huffman coding \cite{huffman1952method}, Arithmetic coding \cite{witten1987arithmetic} and Asymmetric Numeral Systems \cite{duda2013asymmetric}, are developed so that symbols with certain distribution density can be encoded with codelength close to the entropy.
With recent advances in deep generative models, the machine learning community has been actively investigating learned probabilistic models which represent the data distribution more accurately and achieve state-of-the-art performance for lossless compression \cite{townsend2019practical, hoogeboom2019integer, ho2019compression, townsend2019hilloc, mentzer2019practical, berg2020idf++}. 

\noindent{\bf Explicit Probabilistic Modelling} 
Due to the fact that entropy coders require explicit evaluation of density functions, predominant approaches are based on Variational Inference and Maximum Likelihood. 
VAEs leverage the latent variables for expressive explicit mixture models and employ bits-back techniques \cite{wallace1990classification, hinton93keeping} to encode the hierarchical distribution structure. 
The theoretical codelength with VAEs is exactly the variational lower bound. However, there exists a gap between the lower bound and the true data entropy, resulting in inferior performance on compression ratio.
L3C \cite{mentzer2019practical} augments the data variable with hierarchical latents, and the compression ratio is compromised due to the introduction of latent variables. 
Autoregressive models \cite{salimans2017pixelcnn++} stipulates probabilistic models in a fully observable way in which the probability is factorised on the pixel level. They perform well for image generation, but suffer from high computational costs.
Compared to the previous explicit probability models, flow models admit direct evaluation of data density rather than the variational lower bound, yielding more efficient coding schemes in theory.

\noindent{\bf Coding with Flow Models} 
The idea of normalizing flows is based on change of variables, where the diffeomorphism would introduce a push forward map on the original distribution.
In order to have a simple Jacobian matrix determinant evaluation in the push forward probability, the diffeomorphisms are usually composed of successive simple invertible layers, e.g. coupling layers.
Even though flows could admit plausible image generation \cite{dinh2016density, kingma2018glow}, the numeric quantization error is one of the main challenges for lossless compression.
Hoogeboom et al \cite{hoogeboom2019integer} approached this problem by proposing a discrete flow for integers (IDF). 
IDF++~\cite{berg2020idf++} introduces several improvements to IDF. 
However, this very design imposes strict discrete constraints as well as strong constraints on coupling layer structure and thus undermines the expressive power of the flow model.
Ho et al \cite{ho2019compression} proposed a local bits-back coding technique (LBB) for more flexible flow models.
Nevertheless, it suffers from a large quantity of encoding/decoding operations on each layer, which greatly affects the efficiency of lossless compression. Full Jocabian computation may reduce coding times, but computing it is intractable.

The technique introduced in this work enables compression with any volume-preserving flows. Volume preserving flows only introduce limited constraints on general flows without hurting much expressive power, and more types of volume-preserving layers can be considered compared to IDF. 
With novel quantization methods, the invertible computation of flows without any numeric error is achieved, while the compression ratio can reach the entropy of distribution estimated with flows. 


\section{Method}
In this section, we introduce iVPF, a Numerical Invertible Volume Preserving Flow for efficient lossless compression. 
Note that the invertible network in a traditional flow $\mathbf{z} = f(\mathbf{x})$ is not exactly a bijective map as numerical errors are encountered in arithmetic computations of floating points~\cite{behrmann2020understanding}.
Although numerical errors are allowed in normal data generation modelling, no error shall occur in a lossless compression task.
iVPF can be built upon any volume preserving flows and is exactly bijective by eliminating floating-point numeral errors in forward and inverse computation, making it feasible for efficient lossless compression by encoding latent variables $\mathbf{z} = f(\mathbf{x})$.

\subsection{Recap: Volume Preserving Flows}
\label{sec:vpf}
Normalising flows are normally constructed between a continuous domain $\mathcal{X}$ and co-domain $\mathcal{Z}$.
Let $f: \mathcal{X} \to \mathcal{Z}$ be a bijective map defined by a flow model. The model distribution $p_X (\mathbf{x})$ of input data $\mathbf{x}$ can be expressed with the latent variable $\mathbf{z}$ such that
\begin{equation}
    \setlength{\abovedisplayskip}{3pt}
    \setlength{\belowdisplayskip}{3pt}
    \log p_X (\mathbf{x}) = \log p_Z(\mathbf{z}) + \log \Big|\frac{\partial \mathbf{z}}{\partial \mathbf{x}} \Big| \quad \mathbf{z} = f(\mathbf{x}),
\end{equation}
where $|\frac{\partial \mathbf{z}}{\partial \mathbf{x}}|$ is the absolute value of the determinant of the Jacobian matrix. 
The flow model is constructed with composition of flow layers such that $f = f_L \circ ... \circ f_2 \circ f_1$. Denoting $\mathbf{y}_{l} = f_l (\mathbf{y}_{l-1})$ and $\mathbf{y}_0 = \mathbf{x}, \mathbf{z} = \mathbf{y}_L$, we have $\log p_X(\mathbf{x}) = \log p_Z(\mathbf{z}) + \sum_{i=1}^L \log |\frac{\partial \mathbf{y}_l}{\partial \mathbf{y}_{l-1}}|$.
The distribution $p_Z(\mathbf{z})$ can be stipulated from any parametric distribution families, e.g., mixture of Gaussians, mixture of logistic distributions, etc. And the parameters of such distributions can be defined manually or learned jointly with flow parameters.
The inverse flow is computed as $f^{-1} = f_1^{-1} \circ ... \circ f_L^{-1}$.
As discussed in Sec. \ref{sec:why_vpf}, general non-volume preserving flows will cause practical issues when used for lossless compression, and so we focus on volume preserving flows.

Volume preserving flows have the property that $|\frac{\partial \mathbf{z}}{\partial \mathbf{x}}| = 1$, and therefore the density of $\mathbf{x}$ can be directly computed as $p_X (\mathbf{x}) = p_Z(\mathbf{z})$. 
A volume preserving flow model is a composition of volume preserving layers where $|\frac{\partial \mathbf{y}_{l}}{\partial \mathbf{y}_{l-1}}| = 1$ for any $\mathbf{y}_{l} = f_l (\mathbf{y}_{l-1}), l=1,2,...,L$. 
According to the following section, the volume preserving layers can be constructed from most general flow layers with limited constraints, thus its expressive power is close to the general flow.
Besides, volume preserving flows can use broader types of flow layers than IDF, implying better expressive power than IDF.

Next, we briefly review the three typical building layers adopted in the flow model.

\textbf{(1) Coupling layer.} 
For simplicity of notation, we drop the subscript for layer index in $f_l$, and use $\mathbf{x}$ and $\mathbf{z}$ as layer input and output.
A coupling layer with volume preserving is defined by
\begin{equation}
    \setlength{\abovedisplayskip}{3pt}
    \setlength{\belowdisplayskip}{3pt}
    \mathbf{z} = [\mathbf{z}_a, \mathbf{z}_b] = f(\mathbf{x}) = [\mathbf{x}_a, \mathbf{x}_b \odot \exp(\mathbf{s}(\mathbf{x}_a)) + \mathbf{t} (\mathbf{x}_a)], \\
    \label{eq:coupling}
\end{equation}
where $\mathbf{x} = [\mathbf{x}_a, \mathbf{x}_b]$, $\odot$ denotes element-wise multiplication of vectors, and the functions $\mathbf{s}(\cdot)$ and $\mathbf{t}(\cdot)$ can be modelled with neural networks so that
\begin{equation}
    \setlength{\abovedisplayskip}{3pt}
    \setlength{\belowdisplayskip}{3pt}
    \sum_{i} \mathbf{s}_i(\mathbf{x}_a) = 0
    \label{eq:vp_constraint_coupling}
\end{equation}
is the volume preserving constraint in which $\mathbf{s}_i$ is the $i$th element of $\mathbf{s}$.
Note that in IDF, all elements of $\mathbf{s}(\mathbf{x}_a)$ are set to be zero, conveying the limited expressive power.

If the dimension of $\mathbf{z}_a$ is zero, the outputs of $\mathbf{s}(\mathbf{x}_a), \mathbf{t}(\mathbf{x}_a)$ are learned linear coefficients irrelevant with $\mathbf{x}$, and it degrades to \textbf{invertible linear layer}.

\textbf{(3) Invertible $1 \times 1$ convolution.} This layer can be converted to matrix multiplication along channels~\cite{kingma2018glow}. Let $\mathbf{W}$ be the weight of a $1 \times 1$ convolution layer. It is the square matrix with $|\mathrm{det}(\mathbf{W})| = 1$ to ensure volume preserving.

In practice, $\mathbf{W}$ can be parameterised with LU decomposition as 
\begin{equation}
    \setlength{\abovedisplayskip}{3pt}
    \setlength{\belowdisplayskip}{3pt}
    \mathbf{W} = \mathbf{P} \mathbf{L} \mathbf{\Lambda} \mathbf{U}
\end{equation}
where $\mathbf{P}$ is the permutation matrix, $\mathbf{L}, \mathbf{U}$ are the lower and upper triangular matrix respectively whose diagonal are all ones, and $\mathbf{\Lambda}$ is the diagonal matrix with $|\mathrm{det}(\mathbf{\Lambda})| = 1$. For simplicity, we usually use $\mathbf{\Lambda} = \mathbf{I}$.

If $\mathbf{L}, \mathbf{U}$ are both identity matrices, then the $1 \times 1$ convolution layer reduces to \textbf{permutation layer}.

\textbf{(4) Factor-out layer.} To reduce the computation cost, the latent $\mathbf{z}$ can be modelled with a factored model $p(\mathbf{z}) = p(\mathbf{z}_L) p(\mathbf{z}_{L-1} | \mathbf{z}_{L}) ... p(\mathbf{z}_1 | \mathbf{z}_2, ..., \mathbf{z}_L)$, which can be modelled with volume preserving flow. We give an example of $L=2$~\cite{dinh2016density,hoogeboom2019integer} such that 
\begin{equation}
    \setlength{\abovedisplayskip}{3pt}
    \setlength{\belowdisplayskip}{3pt}
    [\mathbf{z}_1, \mathbf{y}_1] = f_1(\mathbf{x}), \quad \mathbf{z}_2 = f_2(\mathbf{y}_1), \quad \mathbf{z} = [\mathbf{z}_1, \mathbf{z}_2]
\end{equation}
where $f_1, f_2$ are both volume preserving flows. Then the probability of $\mathbf{x}$ is given by $p(\mathbf{x}) = p(\mathbf{z}_2) p(\mathbf{z}_1 | \mathbf{y}_1)$. 

\subsection{Volume Preserving Flows in Discrete Space}
\label{sec:why_vpf}

With normalizing flows, the input data may be encoded by first converting input data to latents, and then encoding latents with prior distribution. As all data are stored in binary formats, the data and latents should be discrete. However, normalizing flows naturally work for continuous distributions. In what follows, quantization is needed to make the latents discrete, and discrete bijections between input data and latents should be established. We show that volume preserving flows have the potential to hold the property.

\subsubsection{$k$-precision Quantization}
\label{sec:quantization}

In general practice, the continuous space is discretized to bins with volume $\delta$, where certain continuous data $\mathbf{x}$ is quantized to the centre of a certain bin $\Bar{\mathbf{x}}$.
The probability mass function is approximated by $P(\Bar{\mathbf{x}}) \approx p_X(\Bar{\mathbf{x}}) \delta$ for the corresponding bin.
Specifically, we use a similar quantization strategy as in IDF~\cite{hoogeboom2019integer} where the floating points are quantized to decimals with certain precision $k$ as
\begin{equation}
    \setlength{\abovedisplayskip}{3pt}
    \setlength{\belowdisplayskip}{3pt}
    \round{x} = \frac{\mathrm{round}(2^k \cdot x)}{2^k},
    \label{eq:quant}
\end{equation}
where $\mathrm{round}$ is the rounding operation. 
We refer to such quantization scheme as \textit{k-precision quantization}.
It is worth noting that with large $k$, $x$ and $\round{x}$ can be very close, since $|x - \round{x}| \le 2^{-k-1}$. 
And the bin volume can be computed as $\delta = 2^{-kd}$, where $d$ is the dimension of $\mathbf{x}$.

Due to the discrete nature of electronic data, e.g., 8 bit images, for any specific dataset to compress, there exists an integer $K$, such that for any $k>K$, $k$-precision quantization is sufficient to represent all possible data values.
As a result, we assume that $k$ is chosen large enough s.t. $\round{\mathbf{x}}=\mathbf{x}$ for the input data.
Moreover, it is worth noting that the same $k$-precision quantization scheme is employed for the data $\mathbf{x}$ and all of the latent variables output by each flow layer. And given discrete latents $\bar{\mathbf{z}} = \round{f({\mathbf{x}})}$, the inverse flow must recover the inputs such that $f^{-1}(\bar{\mathbf{z}}) \equiv \mathbf{x}$.

For a chosen $k$-precision quantization, the theoretical codelength of ${\mathbf{x}}$ is $-\log (p_X({\mathbf{x}}) \delta)$. 
With change of variables in flow, ${\mathbf{x}}$ is coded by coding $\bar{\mathbf{z}}$ with theoretical length $-\log (p_Z(\bar{\mathbf{z}}) \delta)$, where $p_Z(\bar{\mathbf{z}}) = p_X({\mathbf{x}}) |\frac{\partial \mathbf{x}}{\partial \mathbf{z}}(\bar{\mathbf{z}})|$.


\subsubsection{Invertibility of Volume Preserving Flows}
\label{sec:propositions}
The essential difference of non-volume preserving flows from volume preserving flows is that the absolute Jacobian determinant is not constant 1 for all data points.
We will see that this characteristic of non-volume preserving flows can lead to either waste of bits for compression or failure of bijection induction that is needed for lossless compression with flows.
As a result, volume preserving flows are preferred for efficient lossless compression.

From Proposition~\ref{cl:long}, for a data point $\mathbf{x}_0$ with $|\frac{\partial \mathbf{z}}{\partial \mathbf{x}}(\mathbf{x}_0)| > 1$, the codelength by coding in the latent space can be longer than it needs to be which leads to suboptimal compression.
From Proposition~\ref{cl:inv}, for a data point $\mathbf{x}_0$ with $|\frac{\partial \mathbf{z}}{\partial \mathbf{x}}(\mathbf{x}_0)| < 1$, the bijection between domain and co-domain with the same discretization scheme cannot be guaranteed, which means an encoded data point cannot be correctly decoded.
The proofs of two propositions can be seen in the Appendix \ref{app:proof_inv}.

\begin{proposition}\label{cl:long}
Let $f$ be a smooth bijection from $\mathcal{X}$ to $\mathcal{Z}$.
Assume for $\mathbf{x}_0$, it holds that $|\frac{\partial \mathbf{z}}{\partial \mathbf{x}}(\mathbf{x}_0)| > 1$.
Then there exists an integer $K$, for any $k$-precision discretisation scheme where $k>K$, we have $-\log (p_Z(\round{f(\mathbf{x}_0)})\delta) > -\log (p_X(\mathbf{x}) \delta)$. 
\end{proposition}

\begin{proposition}\label{cl:inv}
Let $f$ be a smooth bijection from $\mathcal{X}$ to $\mathcal{Z}$.
Assume for $\mathbf{x}_0$, it holds that $|\frac{\partial \mathbf{z}}{\partial \mathbf{x}}(\mathbf{x}_0)| < 1$.
Then there exists an integer $K$, for any $k>K$, $f$ cannot induce a bijection between the discretised domain $\bar{\mathcal{X}}$ and discretised co-domain $\bar{\mathcal{Z}}$.
\end{proposition}


\subsection{Numerical Invertible Volume Preserving Flow Layer}
\label{sec:ivpf}
The previous section illustrated the necessity of efficient coding with volume preserving flows. 
In this subsection, we introduce Numerical Invertible Volume Preserving Flows (iVPF), which are derived from existing volume preserving flows with $k$-precision quantization and novel arithmetic computation algorithms of flow layers.
Denote by $\Bar{f}$ the iVPF derived from the volume preserving flow $f$.
An iVPF model is constructed with composition of numerical invertible volume preserving flow layers such that $\Bar{f} = \Bar{f}_L \circ ... \circ \Bar{f}_1$, where $\mathbf{y}_{l-1} \equiv \Bar{f}_l^{-1} (\Bar{f}_l(\mathbf{y}_{l-1}))$. 
It holds that $\mathbf{x} \equiv \Bar{f}^{-1}(\Bar{f} (\mathbf{x}))$ for all quantized input data $\mathbf{x}$ with the small error between $\Bar{\mathbf{z}} = \Bar{f} (\mathbf{x})$ and $\mathbf{z} = f(\mathbf{x})$, therefore the efficient coding algorithm with iVPF is expected. 
As discussed below, the inputs and outputs of each flow layer in iVPF should be quantized such that $\mathbf{y}_l \equiv \round{\mathbf{y}_l}$ for all $l=0,1,...,L$, and so as $\mathbf{z}$. 

\begin{algorithm}[t]
\small
\caption{Modular Affine Transformation (MAT)}
\textbf{Forward MAT}: $\mathbf{z}_b = \bar{f}(\mathbf{x}_b) = \mathbf{s} \odot \mathbf{x}_b + \mathbf{t}$ given $(\mathbf{x}_b, r)$.

\begin{algorithmic}[1]
\STATE $\mathbf{x}_b \gets 2^k \cdot \mathbf{x}_b$;
\STATE Set $m_{0} = m_{d_b} = 2^C$, assert $0 \le r < m_{0}$;
\FOR {$i = 1,2,...,d_b-1$}
\STATE $m_i \gets \mathrm{round} (m_0 / (\prod_{j=1}^i s_j) )$; $\quad // s_i \approx m_{i-1} / m_{i}$
\ENDFOR
\FOR {$i = 1,2,..., d_b$}
\STATE $v \gets x_{i} \cdot m_{i-1} + r$;
\STATE $y_i \gets \lfloor v / m_i \rfloor, \quad r \gets v \mod m_i$; $\quad // y_i \approx x_i \cdot s_i $
\ENDFOR
\STATE $\mathbf{y}_b = [y_1, ..., y_{d_b}]^\top$;
\STATE $\mathbf{z}_b \gets \mathbf{y}_b / 2^k + \round{\mathbf{t}}$.
\RETURN $\mathbf{z}_b, r$.
\end{algorithmic}

\vspace{6pt}

\textbf{Inverse MAT}: $\mathbf{x}_b = \bar{f}^{-1}(\mathbf{z}_b) =  (\mathbf{z}_b - \mathbf{t}) \oslash \mathbf{s}$ given $(\mathbf{z}_b, r)$.

\begin{algorithmic}[1]
\STATE $\mathbf{y}_b \gets 2^k \cdot (\mathbf{z}_b - \round{\mathbf{t}})$;
\STATE Set $m_{0} = m_{d_b} = 2^C$, assert $0 \le r < m_{0}$;
\FOR {$i = 1,2,...,d_b-1$}
\STATE $m_i \gets \mathrm{round} (m_0 / (\prod_{j=1}^i s_j) )$; $\quad // s_i \approx m_{i-1} / m_{i}$
\ENDFOR
\FOR {$i = d_b, ..., 2, 1$}
\STATE $v \gets y_{i} \cdot m_{i} + r$;
\STATE $x_i \gets \lfloor v / m_{i-1} \rfloor, r \gets v \mod m_{i-1}$; $ // x_i \approx y_i / s_i $
\ENDFOR
\STATE $\mathbf{x}_b \gets \mathbf{x}_b / 2^k$.
\RETURN $\mathbf{x}_b, r$.
\end{algorithmic}

\label{alg:coupling}
\end{algorithm}

\subsubsection{Numerical Invertible Coupling Layer with Modular Affine Transformation}
\label{sec:coupling}

Without loss of generality, denote by $\mathbf{s} = \exp(\mathbf{s}(\mathbf{x}_a)), \mathbf{t} = \mathbf{t} (\mathbf{x}_a)$. Given $\mathbf{x}_b \in \mathbb{R}^{d_b}$, $\mathbf{z}_b \in \mathbb{R}^{d_b}$ is computed by affine transformation such that $\mathbf{z}_b = \mathbf{s} \odot \mathbf{x}_b + \mathbf{t}$, and the inverse is $\mathbf{x}_b =  (\mathbf{z}_b - \mathbf{t}) \oslash \mathbf{s}$\footnote{$\oslash$ denotes element-wise division.}.
It is clear that traditional multiplication and division do not ensure invertibility as $\round{(\round{\mathbf{s} \odot \mathbf{x}_b + \mathbf{t}} - \mathbf{t}) \oslash \mathbf{s}} \neq \mathbf{x}_b$.
Denote by $z_i, x_i, s_i$ and $t_i$ the $i$th element of $\mathbf{z}_b, \mathbf{x}_b, \mathbf{s}$ and $\mathbf{t}$ respectively.
Note that $\prod_{i=1}^{d_b} s_i = 1$ in iVPF. 
By introducing an auxiliary variable $r$, we give a novel numerical invertible affine transformation algorithm shown in Alg. \ref{alg:coupling}. 
Since the numerical invertibility of the affine transformation is achieved by modular arithmetic, we refer to the induced numerical invertible affine transformation as \textit{Modular Affine Transformation} (MAT). 

\begin{figure}[t]
\begin{center}
\includegraphics[width=0.45\textwidth]{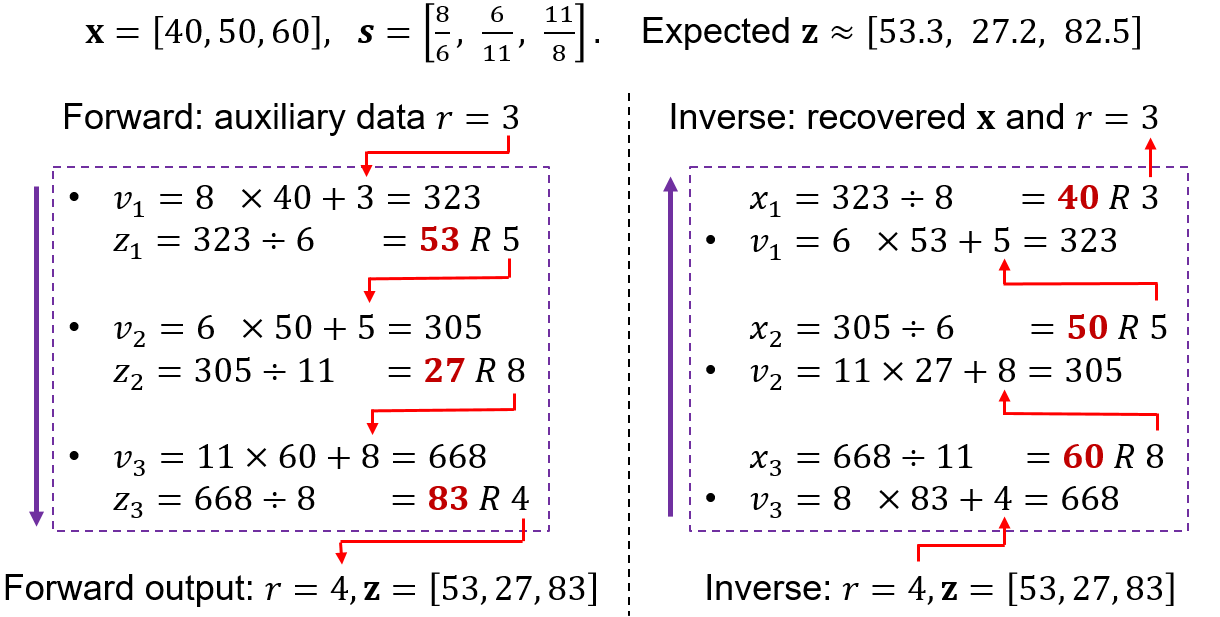}\\
\end{center}
\caption{Illustration of how MAT works with Alg. \ref{alg:coupling}. $a \div b = q \ R \ r$ denotes the division with remainder such that $a = b \cdot q + r$. For simplicity, the inputs are rescaled to integers as Line 1.}
\label{fig:divide}
\end{figure}

The proof of inveribility of MAT is shown in Appendix~\ref{app:proof_mat}. Since the input to both forward and inverse MAT are $k$-precision floating numbers, they are converted to integers in Line 1. 
In Line 2-3, $s_i$ is approximated with fraction $m_{i-1} / m_i$, and therefore $m_{d_b} = m_0$ as the volume preserving constraint holds. 
It is worth noting that $\mathrm{round}$ is the normal rounding operation except for outputting 1 for the input in $(0,1]$ to avoid $m_i$ being 0.
Lines 6-9 are implemented by division with remainders. In Line 7, $r \in [0, m_{i-1})$, so we have $y_i = \lfloor (x_i \cdot m_{i-1} + r) / m_i \rfloor \approx s_i \cdot x_i$ and $x_i = \lfloor (y_i \cdot m_i + r) / m_{i-1} \rfloor \approx x_i / s_i (r \in [0, m_i))$. Formally, the error of computation of $\mathbf{z}_b$ with Alg. \ref{alg:coupling} is
\begin{equation}
    \setlength{\abovedisplayskip}{3pt}
    \setlength{\belowdisplayskip}{3pt}
    |z_i - (s_i x_i + t_i)| = O(2^{-k}, 2^{-C})
\label{eq:coupling_error}
\end{equation}
and the detailed derivation of Eq. (\ref{eq:coupling_error}) is shown in Appendix~\ref{app:error}. 
A simple illustration on the correctness of the algorithm is shown in Fig. \ref{fig:divide}. 

To reduce error, we set $C=16$ (in Line 3). 
Note that an auxiliary value $r$ is involved in the algorithm. 
When using the iVPF model, we can simply initialise $r=0$, update $r$ at each coupling layer and finally store $r$. 
When using the inverse iVPF model, the stored $r$ is fed into the model and is updated at each inverse coupling layer. When the original data are restored, $r$ is reduced to zero. As $r \in [0, 2^C)$, storing $r$ just needs additional $C$ bits, which is negligible in compressing high-dimensional data.

\subsubsection{Numerical Invertible $1 \times 1$ Convolution Layer}
\label{sec:conv}

Let $\mathbf{x}, \mathbf{z} \in \mathbb{R}^c$ be input and output of the convolution layer along the feature map at a certain pixel.
Then we have $\mathbf{z} = \mathbf{W} \mathbf{x} = \mathbf{P} \mathbf{L} \mathbf{\Lambda} \mathbf{U} \mathbf{x}$. 

\textbf{(1) Matrix multiplication with upper triangular matrix $\mathbf{U}$.} 
Denote by $x_i$ and $z_i$ the $i$th element of $\mathbf{x}$ and $\mathbf{z}$ respectively.
And let $u_{ij}$ be the elements of $\mathbf{U}$. 
In iVPF, the forward mode matrix multiplication is $z_i = x_i + \round{\sum_{j=1+1}^c u_{ij} x_j}, i = 1,2,...,c-1$ and $z_c = x_c$. And the inverse matrix multiplication is computed in an auto-regressive manner in which $x_c, x_{c-1}, ..., x_1$ are computed recursively such that
\begin{equation}
\setlength{\abovedisplayskip}{3pt}
\setlength{\belowdisplayskip}{3pt}
\begin{split}
    x_c &= z_c, \\
    x_i &= z_i - \round{\sum_{j=i+1}^c u_{ij} x_j}, \quad i = c-1, ..., 1.
\end{split}
\label{eq:upper}
\end{equation}

\textbf{(2) Matrix multiplication with lower triangular matrix $\mathbf{U}$.} 
Denote by $l_{ij}$ the elements of $\mathbf{L}$.
Then the forward mode matrix multiplication is $z_1 = x_1, z_i = x_i + \round{\sum_{j=1}^{i-1} l_{ij} x_j}, i = 2,...,c$. And the inverse mode is performed by recursively computing $x_1, x_{2}, ..., x_c$ such that
\begin{equation}
\setlength{\abovedisplayskip}{3pt}
\setlength{\belowdisplayskip}{3pt}
\begin{split}
    x_1 &= z_1, \\
    x_i &= z_i - \round{\sum_{j=1}^{i-1} l_{ij} x_j}, \quad i = 2,...,c.
\end{split}
\label{eq:lower}
\end{equation}

\textbf{(3) Matrix multiplication with $\mathbf{\Lambda}$.} Denote $\lambda_i$ as the $i$-th diagonal element of $\mathbf{\Lambda}$, we have $z_i = \lambda_i x_i$. It is clear that $|\prod_{i=1}^c \lambda_i| = 1$, thus we use MAT shown in Alg. \ref{alg:coupling}.

\textbf{(4) Matrix multiplication with $\mathbf{P}$.} As $\mathbf{P}$ is the permutation matrix, the elements in $\mathbf{x}$ can be directly permuted to $\mathbf{z}$, and $\mathbf{x}$ can be recovered by applying the inverse permutation $\mathbf{P}^{-1}$ on $\mathbf{z}$.

\begin{figure*}[t]
\begin{center}
\includegraphics[width=0.85\textwidth]{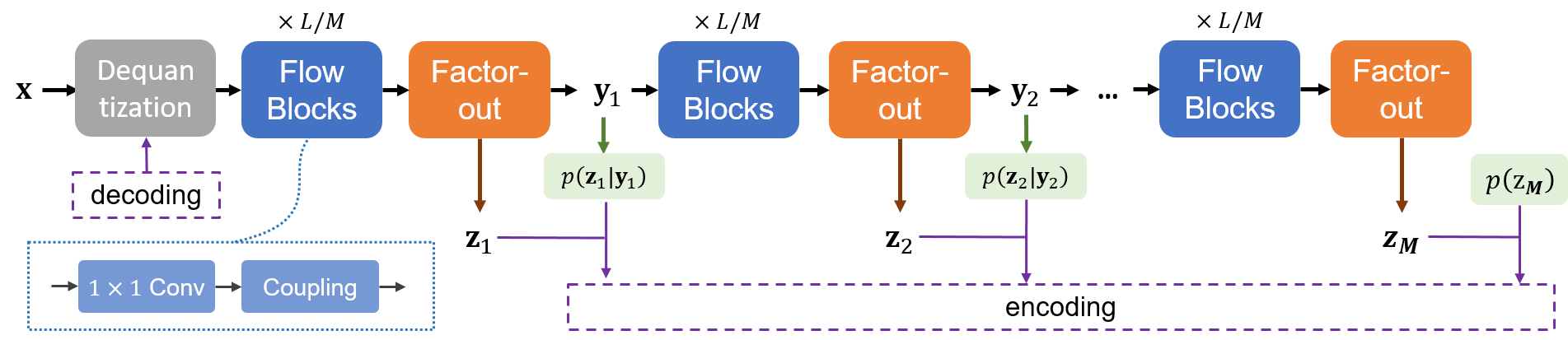}\\
\end{center}
\caption{Illustration of the iVPF architecture. We just show the encoding phase with iVPF model and the decoding phase is performed inversely. The latents in the model are $\mathbf{z} = [\mathbf{z}_1, ..., \mathbf{z}_M]$ and the density function is $p(\mathbf{z}) = p(\mathbf{z}_M) \prod_{l=1}^{M-1} p(\mathbf{z}_l | \mathbf{y}_l)$.}
\label{fig:flow}
\end{figure*}

\subsection{Error Analysis}

Consider the continuous flow layer $f_l$ and the corresponding iVPF layer $\bar{f}_l$. 
With the same input $\mathbf{y}_{l-1}$, for coupling layer, $|\bar{f}_l(\mathbf{y}_{l-1}) - f_l(\mathbf{y}_{l-1})| = O(2^{-k}, 2^{-C})$; for other layers, $|\bar{f}_l(\mathbf{y}_{l-1}) - f_l(\mathbf{y}_{l-1})| = O(2^{-k})$. 
If we further assume $f$ is Lipschitz continuous, it holds that $|f_l(\bar{\mathbf{y}}_{l-1}) - f_l(\mathbf{y}_{l-1})| = O(|\bar{\mathbf{y}}_l - \mathbf{y}_l|)$. Thus for the volume preserving flow model with $\mathbf{y}_0 = \mathbf{x}, \mathbf{z} = \mathbf{y}_L, \mathbf{y}_l = f_l (\mathbf{y}_{l-1})$, and the corresponding iVPF model with $\bar{\mathbf{y}}_0 = \mathbf{x}, \bar{\mathbf{z}} = \bar{\mathbf{y}}_L, \bar{\mathbf{y}}_l = \bar{f}_l (\bar{\mathbf{y}}_{l-1})$, we have $|\bar{\mathbf{y}}_l - \mathbf{y}_l| = |\bar{\mathbf{y}}_{l-1} - \mathbf{y}_{l-1}| + O(2^{-k}, 2^{-C})$ and therefore $|\bar{\mathbf{z}} - \mathbf{z}| = O(L2^{-k}, L2^{-C})$. Detailed derivation can be found in the supplementary material. If $L, 2^{-k}, 2^{-C}$ are relatively small, the latents generated with the iVPF model are able to approximate the true distribution.

\begin{algorithm}[t]
\small
\caption{Coding with iVPF}
\textbf{Procedure} ENCODE($\mathbf{x}_0$)

\begin{algorithmic}[1]
\STATE Pre-process: $\mathbf{x} \gets \round{\mathbf{x}_0 / 2^h - 0.5}$, set $r=0$;
\STATE Decode $\mathbf{u} \sim U(0, 2^{-h}) \delta$;
\STATE $\bar{\mathbf{x}} \gets \mathbf{x} + \mathbf{u}$;
\STATE Compute $\bar{\mathbf{z}} \gets \bar{f}(\bar{\mathbf{x}})$ with the iVPF model, in which $r$ is updated with forward MAT;
\STATE Encode $\bar{\mathbf{z}} \sim p_Z(\bar{\mathbf{z}}) \delta$; store $r$.
\end{algorithmic}

\vspace{8pt}

\textbf{Procedure} DECODE()

\begin{algorithmic}[1]
\STATE Decode $\bar{\mathbf{z}} \sim p_Z(\bar{\mathbf{z}}) \delta$; get $r$.
\STATE Compute $\bar{\mathbf{x}} \gets \bar{f}^{-1}(\bar{\mathbf{x}})$ with the inverse iVPF model, in which $r$ is updated with inverse MAT;
\STATE $\mathbf{x} \gets \lfloor 2^k \cdot \bar{\mathbf{x}} \rfloor / 2^h, \quad \mathbf{u} \gets \bar{\mathbf{x}} - \mathbf{x}$;
\STATE Encode $\mathbf{u} \sim U(0, 2^{-h})\delta$;
\STATE Post-process: $\mathbf{x}_0 \gets 2^h \cdot (\mathbf{x} + 0.5)$;
\RETURN $\mathbf{x}_0$.
\end{algorithmic}

\label{alg:ivpf}
\end{algorithm}

\subsection{Bits-Back Coding with Dequantization}

Consider the case when elements of the $d$-dimensional input data are all integers, i.e., $\mathbf{x}_0 \in \{0, 1, ..., 2^h-1\}^d$ (for byte data or 8-bit images, $h=8$). In practical applications, $\mathbf{x}_0$ should be pre-processed by normalisation and quantization such that $\mathbf{x} = \round{\mathbf{x}_0 / 2^h - 0.5}$. 
To avoid information loss during quantization, it is set that $k \ge h$. 
After normalisation, the encoding process is (1) getting latent variable $\bar{\mathbf{z}} = \Bar{f}(\mathbf{x})$ with iVPF model; (2) encoding $\bar{\mathbf{z}}$ with prior $P(\bar{\mathbf{z}}) \approx p_Z(\bar{\mathbf{z}}) \delta$ using rANS coding algorithm~\cite{duda2013asymmetric}; and (3) saving auxiliary data generated by the iVPF model (see Sec. \ref{sec:coupling}). 
The decoding process is (1) loading auxiliary data; (2) decoding $\bar{\mathbf{z}}$ with $P(\bar{\mathbf{z}}) \approx p_Z(\bar{\mathbf{z}}) \delta$; and (3) recovering $\mathbf{x}$ with the inverse iVPF given $\bar{\mathbf{z}}$. 
In the encoding and decoding process, the volume of quantization bins is $\delta = 2^{-kd}$.
As a result, the expected entropy coding length of $\mathbf{x}$ is approximately $-\log p_Z(\bar{\mathbf{z}}) + kd + C$, where $C=16$ denotes the bit length of the auxiliary data (see Sec. \ref{sec:coupling}).

However, the coding length with the above coding technique is significantly affected by the decimal precision $k$. 
If $k > h$, the coding length is dramatically increased with larger $k$. 
If $k (= h)$ is relatively small, the error between $\Bar{\mathbf{z}} = \Bar{f} (\mathbf{x})$ and the latents with continuous volume preserving flow $\mathbf{z} = f(\mathbf{x})$ are large.
Then the density estimation of $p_X(\mathbf{x})$ will be far from the true distribution. 
In this case, the expected codelength will be large.

To resolve this issue, we adopt the idea of LBB~\cite{ho2019compression} by dequantizing $\mathbf{x}$ to certain decimal precision $k$ with the bits-back coding scheme. 
By using $k-h$ auxiliary bits, we can decode $\mathbf{u} \in \mathbb{R}^d$ from the uniform distribution $P(\mathbf{u} | \mathbf{x}) = U(0, 2^{-h}) \delta$\footnote{Entries in $\mathbf{x}_0$ are integers, and $\mathbf{x} = \round{\mathbf{x}_0 / 2^h - 0.5}$ is discretized with bin size $2^{-h}$.} with quantization bin size $\delta = 2^{-kd}$, and put $\mathbf{x} + \mathbf{u}$ in the iVPF model such that $\bar{\mathbf{z}} = \bar{f}(\mathbf{x} + \mathbf{u})$. Then expected codelength of $\mathbf{x}$ is 
\begin{equation}
\setlength{\abovedisplayskip}{4pt}
\setlength{\belowdisplayskip}{4pt}
\begin{split}
    &\mathbb{E}_{q(\mathbf{u} | \mathbf{x})} [- \log (p_X(\mathbf{x} + \mathbf{u}) \delta) + \log (q(\mathbf{u | x}) \delta)] + C \\
    =& \mathbb{E}_{q(\mathbf{u} | \mathbf{x})} [- \log p_Z(\bar{\mathbf{z}}) ] + hd + C
\end{split}
\label{eq:codelength}
\end{equation}
which is irrelevant with the decimal precision $k$. Thus we can set higher $k$ to ensure accurate density estimation of $\mathbf{x}$. 

The overall lossless compression algorithm with iVPF is summarised in Alg. \ref{alg:ivpf}. 



\begin{figure*}[t]
\begin{center}
\includegraphics[width=0.85\textwidth]{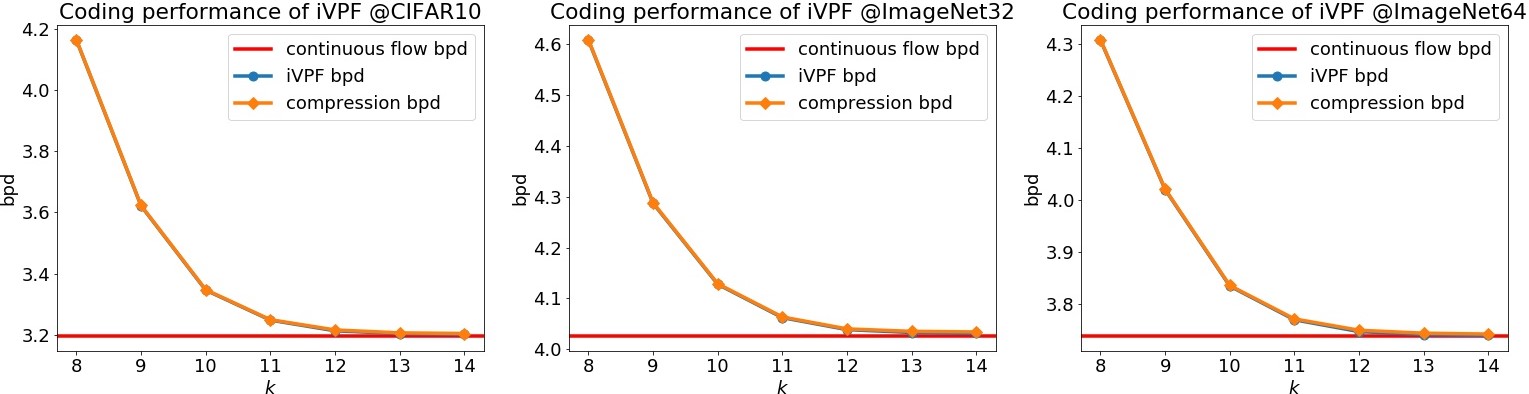}\\
\end{center}
\caption{Compression performance of iVPF in terms of bpd. The negative log-likelihood (the lower bound of compression length) of continuous flow and the iVPF are also shown. Orange lines with diamond overlaps with the blue with circle ones as they are very close.}
\label{fig:compression}
\end{figure*}

\section{Experiments} 

In this section, we perform extensive experiments to evaluate the effectiveness of lossless compression with iVPF model. 
We first introduce the continuous volume preserving flow architecture used for iVPF, and test the performance of iVPF on both low resolution images and real-world high resolution ones.

\subsection{Flow Architecture}
\label{sec:flow_arch}

The architecture is a continuous flow derived from IDF~\cite{hoogeboom2019integer} and IDF++~\cite{berg2020idf++}, which is shown in Fig. \ref{fig:flow}. The architecture consists of $L$ stacked flow layer blocks, which are split into $M$ levels. 

The flow input is normalised to $[-0.5, 0.5]$. Each flow block is composition of volume preserving flow layers including $\mathrm{Coupling-1\times 1 Conv}$. 
For $\mathrm{Coupling}$ layer, we follow ~\cite{berg2020idf++} and split the input by $3:1$ where the dimension of $\mathbf{x}_b$ is $d_b = d/4$, while $1:1$ split reaches similar results. $\mathbf{s}(\mathbf{x}_a)$ and $\mathbf{t}(\mathbf{x}_a)$ are predicted with DenseNets~\cite{huang2017densely}, and the weights are shared except for the last layer. 
The dense blocks in the Densenet are the same as IDF++ in which Swish activation~\cite{howard2019searching} and group norm~\cite{wu2018group} are used. 
Denote by $\mathbf{s}_0$ and $\mathbf{t}_0$ the outputs of the DenseNet, the final coupling coefficients are computed by $\mathbf{s}(\mathbf{x}_a) = \alpha [\tanh (\mathbf{s}_0) - \mathrm{mean} (\tanh(\mathbf{s}_0))], \mathbf{t}(\mathbf{x}_a) = \alpha \mathbf{t}_0$, where $\alpha$ is a learned parameter and $\mathrm{mean}$ denotes the mean of vector elements to make the flow volume-preserving. 
We initialise $\alpha$ to be zero so that the coupling layer starts with identity transformation (see Eq. (\ref{eq:coupling})) for training~\cite{berg2020idf++}.
$\tanh$ is used to constrain the scale for stability. 
For $\mathrm{1 \times 1 Conv}$ layer, the weights are decomposed to $\mathbf{W} = \mathbf{P} \mathbf{L} \mathbf{U}$.
$\mathbf{P}$ is the permutation matrix which is initialised once and kept fixed; $\mathbf{L}$ and $\mathbf{U}$ are lower and upper triangular matrices respectively to learn.

At the end of each level, a factor-out layer is introduced, except for the final level. The architecture of factor-out layer is exactly the same as IDF++~\cite{berg2020idf++}.

In the training phase, the continuous flow is trained with variational dequantization to make the input data continuous~\cite{hoogeboom2020learning}. To be specific, a dequantization noise $\mathbf{u}$ is sampled with $q(\mathbf{u} | \mathbf{x}) = U(0, 2^{-h})$ and the training objective is minimising $\mathbb{E}_{q(\mathbf{u} | \mathbf{x})} [- \log p_X(\mathbf{x} + \mathbf{u}) + \log q(\mathbf{u | x})]$. After training, The iVPF model is derived from the continuous flow, in which the flow layer are computed with Sec. \ref{sec:ivpf}, and rANS~\cite{duda2013asymmetric} is adopted for encoding and decoding. Mindspore tool\footnote{\url{https://www.mindspore.cn/}} is used for implementation.

\subsection{Dataset and Training Details}

Like general lossless compression algorithms~\cite{hoogeboom2019integer,berg2020idf++,townsend2019hilloc,ho2019compression}, we use CIFAR10, ImageNet32 and ImageNet64 for training and evaluation. 
For the CIFAR10 and ImageNet32 datasets, the number of flow blocks is $L=24$ with $M=3$ levels, while in ImageNet64, $L=32$ and $M=4$. 
The depth of Densenet is $12$ in each coupling and factor-out layer. 
Adamax~\cite{kingma2014adam} is used for training, and the initial learning rate of each dataset is $2 \times 10^{-3}$ and gradually decreasing after each epoch. 
We train approximately 500 epochs for CIFAR10, 70 epochs for ImageNet32 and 20 epochs for ImageNet64. For better convergence, we adopt exponential moving average with decay rate $0.9995$ on model weights to get ensembled model for evaluation. 

To test the generation performance, we evaluate the above datasets on ImageNet32 trained model, in which ImageNet64 is split into four $32 \times 32$ patches. We also evaluate compression on natural high-resolution image datasets including CLIC.mobile, CLIC.pro\footnote{\url{https://www.compression.cc/challenge/}} and DIV2k~\cite{agustsson2017ntire}. These images are cropped to $32 \times 32$ patches and feed into the ImageNet32 trained model for evaluation. We also feed $64 \times 64$ patches to ImageNet64 trained model for comparison.

The evaluation protocol is the average bits per dimension (bpd), which can be derived with the negative log-likelihood divided by the dimension of the inputs. For no compression, the bpd is 8. We also investigate the coding efficiency of the proposed algorithm.

\subsection{Compression Performance}
\label{sec:ivpf_performance}

In this subsection, we investigate the compression performance of the proposed iVPF model on different datasets and different settings. 
We first evaluate the compression performance in terms of bits per dimension (bpd) under different quantization precision $k$. 
The analytical log-likelihood of continuous volume preserving flow and iVPF model are also reported for comparison. Fig. \ref{fig:compression} shows the results. 
It can be seen that the compression codelength is very close to the analytical log-likelihood of the original volume preserving flow model. To be specific, in CIFAR10, when $k=14$, the difference of bpd between the iVPF model and the continuous volume preserving flow is $0.005$, which is the bpd of the auxiliary data introdued in Alg. \ref{alg:coupling} ($16 / 3072 \approx 0.005$). While the final entropy coding only introduces another $0.003$ bpd compared with the iVPF model. Similar conclusions can be arrived on ImageNet32 and ImageNet64 dataset. We therefore set $k=14$ in the rest of the section.

\begin{table}[t]
\centering
\small
\caption{Coding performance of LBB and iVPF on CIFAR10 in terms of bits per dimension and coding efficiency. $k=14$ is used in iVPF. Infer time denotes the inference time with batch size 256.}
\label{tab:small_lbb}
\begin{tabular}{lccccc}
\toprule
& & & \multicolumn{2}{c}{time (ms)} & \\
         & bpd & aux. bits & infer & coding & \# coding \\
\midrule
LBB~\cite{ho2019compression} & \textbf{3.12} & 39.86 & 13.8 & 97 & 188 \\
\textbf{iVPF (Ours)} & 3.20 & \textbf{6.00} & \textbf{6.1} & \textbf{6.6} & \textbf{2} \\
\bottomrule
\end{tabular}
\end{table}

As compression with iVPF only involves a single decoding and encoding process, model inference contributes the main latency of coding, which is expected to be effective. Table \ref{tab:small_lbb} conveys that compression with iVPF can be performed in milliseconds.
More discussions on the coding time can be found in the supplementary material.
On the other hand, as the coding scheme has to sequentially encode each dimension of data, the latency of coding is negligible when the quantity of coding is large.
LBB~\cite{ho2019compression} is such an algorithm with large quantity of coding operations, and our re-implementation on LBB implies that coding takes about $90\%$ of the compression latency. 
Table \ref{tab:small_lbb} shows the coding performance between LBB and iVPF on CIFAR10 in terms of bpd, auxiliary bits used per dimension and the number of times for coding. Although LBB achieves slightly lower bpd, it suffers from large quantity of coding. Moreover, our algorithm could significantly reduce the auxiliary bits used\footnote{The quantization precision of the original byte data is $h=8$. We use $6$ auxiliary bits for dequantization to make the precision $k=14$.}.
As a result, compared to LBB, our algorithm is competitive with notable space/time efficiency and limited compression ratio loss. 



\begin{table}[t]
\centering
\small
\caption{Compression performance in bits per dimension (bpd) on benchmarking datasets. $^\dag$ denote the generation performance in which the models are trained on ImageNet32 and tested on other datasets.}
\label{tab:small}
\begin{tabular}{lcccc}
\toprule
         & ImageNet32 & ImageNet64 & CIFAR10 & \\
\midrule
PNG \cite{boutell1997png}     & 6.39        & 5.71        & 5.87 \\
FLIF \cite{sneyers2016flif}    & 4.52        & 4.19        & 4.19 \\
\midrule
L3C \cite{mentzer2019practical}      & 4.76        & -           & - \\
Bit-Swap \cite{kingma2019bit} & 4.50        & -           & 3.82 \\
HiLLoC \cite{townsend2019hilloc}$^\dag$  & 4.20        & 3.90        & 3.56 \\
\midrule
LBB~\cite{ho2019compression}      & 3.88        & 3.70        & 3.12 \\
\midrule
IDF \cite{hoogeboom2019integer}     & 4.18        & 3.90        & 3.34 \\
IDF \cite{hoogeboom2019integer}$^\dag$     & 4.18        & 3.94        & 3.60 \\
IDF++ \cite{berg2020idf++}   & 4.12        & 3.81        & 3.26 \\
\textbf{iVPF (Ours)}     & \underline{4.03}        & \underline{3.75}     & \underline{3.20} \\
\textbf{iVPF (Ours)}$^\dag$     & \underline{4.03}        & \underline{3.79}     & \underline{3.49}  \\
\bottomrule
\end{tabular}
\end{table}

\begin{table}[t]
\centering
\small
\caption{Compression performance in bpd of different algorithms on real-world high resolution images. $^\ddag$ denotes compression with ImageNet64 trained model.}
\label{tab:big}
\begin{tabular}{lccc}
\toprule
         & CLIC.mobile & CLIC.pro & DIV2k \\
\midrule
PNG \cite{boutell1997png}     & 3.90 & 4.00 & 3.09 \\
FLIF \cite{sneyers2016flif}    & 2.49 & 2.78 & 2.91 \\
\midrule
L3C \cite{mentzer2019practical}     & 2.64 & 2.94 & 3.09 \\
RC \cite{mentzer2020learning}      & 2.54 & 2.93 & 3.08 \\
\midrule
\textbf{iVPF (Ours)}     & \textbf{2.47} & \textbf{2.63} & \textbf{2.77} \\
\textbf{iVPF (Ours)}$^\ddag$     & \textbf{2.39} & \textbf{2.54} & \textbf{2.68} \\
\bottomrule
\end{tabular}
\end{table}

\subsection{Comparison with State-of-the-Art}
To show the competitiveness of our proposed iVPF model, we compare its performance with state-of-the-art learning based lossless compression methods and representative conventional lossless coders, i.e., PNG \cite{boutell1997png} and FLIF \cite{sneyers2016flif}, on small and large scale datasets.

\textbf{Experiments on Low Resolution Images.} We compare iVPF with L3C \cite{mentzer2019practical}, Bit-Swap \cite{kingma2019bit} and HiLLoC \cite{townsend2019hilloc}, which represent cutting-edge performance of non-flow based latent variable models.
Moreover, we compare the results with a recent advance on lossless compression with continuous flows including LBB \cite{ho2019compression}, IDF~\cite{hoogeboom2019integer} and IDF++~\cite{berg2020idf++}.
The results are shown with bpd to encode one image averaged on the datasets, where lower bpd shows the advantage of the method.
From Table \ref{tab:small}, one can see that our proposed method outperforms its predecessors, i.e., IDF and IDF++, which validates the improvement of our method in the line of integer flows.
Although LBB preforms slightly better than iVPF, our method is competitive as the former suffers from large number of auxiliary bits and low coding efficiency, as discussed in Sec. \ref{sec:ivpf_performance}.

\textbf{Generation.} We test the generation performance of iVPF, with HiLLoc~\cite{townsend2019hilloc}, IDF~\cite{hoogeboom2019integer} and LBB~\cite{ho2019compression} for compression. The results is shown in Table \ref{tab:small}. It can be seen that most compression algorithms achieves desirable generation performance, and ours achieves the best performance under all generation experiments.

\textbf{Experiments on High Resolution Images.} We compare our method with representative practical image lossless compression algorithms, i.e., L3C \cite{mentzer2019practical} and RC \cite{mentzer2020learning}. We use patch-based evaluation in which images are cropped to $32 \times 32$ for ImageNet32 model and $64 \times 64$ for ImageNet64 model.
From Table \ref{tab:big}, one can observe that for high resolution images, conventional coders, e.g., FLIF, still have advantage over learning based methods, while they tend to underperform learning based methods on small scale images.
Nevertheless, our proposed iVPF outperform all the benchmarks shown for high resolution images and thus validates its competitiveness.

\section{Conclusions}

In this paper, we present iVPF, the Numerical Invertible Volume Preserving Flow Model for efficient lossless compression. 
We investigate the potential of volume preserving flow for efficient lossless compression. 
Then we introduce iVPF with novel numeral computing algorithms to the flow model to achieve exact bijective mapping between data and latent variables, so that efficient lossless compression is available by directly coding the latents. The novel lossless compression algorithm based on iVPF is proposed by combining the bits-back coding scheme. Experiments show that the compression codelength reaches the theoretical lower bound of negative log-likelihood. 
The proposed model achieves high compression ratio with desired coding efficiency on image datasets with both low and high resolution images.

{\small
\bibliographystyle{ieee_fullname}
\bibliography{flow.bib}
}

\appendix

\clearpage

\section{Detailed Implementation of iVPF}
\label{app:detail_ivpf}

\subsection{Normalising Flow and Volume Preserving Flows}

Unless specified, in this subsection, we discuss the continuous flow, which is generally used for training. Note that iVPF is constructed based on the trained model with minor modifications.


\subsubsection{Normalising Flows}

As discussed in Sec. \ref{sec:vpf}, the general normalising flow establishes a continuous bijection $f: \mathcal{X} \to \mathcal{Z}$ between input data $\mathcal{X}$ and latent space $\mathcal{Z}$. 
The flow model is constructed with composition of flow layers such that $f = f_L \circ ... \circ f_2 \circ f_1$. 
Let $\mathbf{y}_{l} = f_l (\mathbf{y}_{l-1})$, $\mathbf{y}_0 = \mathbf{x}$ and $\mathbf{z} = \mathbf{y}_L$.
The prior distribution $p_Z(\mathbf{z})$ can be stipulated from any parametric distribution families, e.g., Gaussians, logistic distribution, etc.
Then the model distribution $p_X (\mathbf{x})$ of input data $\mathbf{x}$ can be expressed with the latent variable $\mathbf{z}$ such that
\begin{equation}
\log p_X(\mathbf{x}) = \log p_Z(\mathbf{z}) + \sum_{i=1}^L \log |\frac{\partial \mathbf{y}_l}{\partial \mathbf{y}_{l-1}}|,
\end{equation}
where $|\frac{\partial \mathbf{y}_l}{\partial \mathbf{y}_{l-1}}|$ is the absolute value of the determinant of the Jacobian matrix. 

The inverse flow is $f^{-1} = f_1^{-1} \circ ... \circ f_L^{-1}$. Given latents $\mathbf{y}_L = \mathbf{z}$, the original data $\mathbf{x}$ can be recovered with $\mathbf{y}_{l-1} = f^{-1}_l (\mathbf{y}_l)$ and $\mathbf{x} = \mathbf{y}_0$.

\subsubsection{Constructing Volume Preserving Flows}

Volume preserving flows have the property $|\frac{\partial \mathbf{z}}{\partial \mathbf{x}}| = 1$.
Each volume preserving flow is a composition of volume preserving layers where $|\frac{\partial \mathbf{y}_{l}}{\partial \mathbf{y}_{l-1}}| = 1$ for any $\mathbf{y}_{l} = f_l (\mathbf{y}_{l-1}), l=1,2,...,L$. $p_X (\mathbf{x})$ can be directly computed with prior distribution such that $p_X (\mathbf{x}) = p_Z(\mathbf{z}), \mathbf{z} = f(\mathbf{x})$. 
The prior distribution adopted in this work is more complex than that used in general flows, and the distribution parameters are trainable variables. In this paper, the prior distribution is the mix-Gaussian distribution
\begin{equation}
    p_Z(\mathbf{z}) = \prod_{i=1}^d \big[ \sum_{k=1}^K \pi_{ik} \cdot \mathcal{N} (z_i | \mu_{ik}, \sigma^2_{ik}) \big],
    \label{eq:prior}
\end{equation}
where $\mathbf{z} = [z_1, ..., z_d]^\top$ and $\sum_{k=1}^K \pi_i = 1$. $\pi_{ik}, \mu_{ik}, \sigma^2_{ik}$ are computed with learnable parameters $\bm{\alpha}, \bm{\mu}, \bm{\gamma} \in \mathbb{R}^{d \times K}$. In particular, $\pi_{ik} = \mathrm{softmax} (\bm{\alpha}_i), \mu_{ik} = \bm{\mu}_{ik}, \sigma^2_{ik} = \exp (\bm{\gamma}_{ik})$.

For factor-out layers (shown in Sec. \ref{sec:vpf}), $p(\mathbf{z}_1 | \mathbf{y}_1)$ is modelled with gaussian distribution
\begin{equation}
    p(\mathbf{z}_1 | \mathbf{y}_1) = \mathcal{N}(\mathbf{z}_1 | \bm{\mu}(\mathbf{y}_1), \exp(\bm{\gamma}(\mathbf{y}_1))),
\end{equation}
where $\bm{\mu}(\cdot), \bm{\gamma}(\cdot)$ are all modelled with neural networks. The input dimension is the same as that of $\mathbf{y}_1$ and the output dimension is the same as that of $\mathbf{z}_1$.

From Sec. \ref{sec:vpf}, it can be seen that the volume preserving layers can be constructed from most general flow layers with limited constraints, thus the expressive power is close to the general flow. 
Moreover, the complex prior with learnable parameters improves the expressive power of iVPF. In fact, compared to general flows, iVPF reaches higher bpd than most non-volume preserving ones like RealNVP~\cite{dinh2016density} and Glow~\cite{kingma2018glow}.

\subsubsection{Training Volume Preserving Flows}

As $p_X(\mathbf{x})$ can be directly computed, the training objective is the maximum log-likelihood. Note that we use the discrete data for training the continuous flow. In particular, the input data (images, texts, binary data, etc.) are a branch of integers such that $\mathbf{x} \in \{0, 1, ..., 2^h-1 \}^d$. To resolve this issue, we use variational dequantization technique to make the input data continuous~\cite{ho2019compression} in which some noise is added to input data. Given the discrete data $\mathbf{x}$, the training objective is
\begin{equation}
    \mathcal{L} = -\log p_X(\mathbf{x} + \mathbf{u}^\circ) + \log q (\mathbf{u}^\circ | \mathbf{x}) \quad \mathbf{u}^\circ \sim q(\mathbf{u} | \mathbf{x})
    \label{eq:vdeq}
\end{equation}
where $\mathbf{u}^\circ \in [0, 1)^d$. $q(\mathbf{u} | \mathbf{x})$ is certain distribution within $[0, 1)^d$, which can be either learned distribution with parameters (e.g. flows), or pre-defined distributions. 

In this paper, we simply use uniform distribution such that $q(\mathbf{u} | \mathbf{x}) = U(0,1)$\footnote{Note that the uniform distribution is slightly different than used in the main paper (see Sec. \ref{sec:flow_arch}), as the input data is not normalised here. In fact, they are equivalent after normalisation.}, in which the noise is dependent with the input such that $q(\mathbf{u} | \mathbf{x}) = q(\mathbf{u})$. This implementation is simple for flow training. In fact, we may use more complex $q(\mathbf{u} | \mathbf{x})$ for better model. Moreover, for stable training, we normalise the dequantized data $\mathbf{x} + \mathbf{u}^\circ$ to $[-0.5, 0.5]$ by linear transformation before feeding the flow.

It is clear that the expected value of Eq. (\ref{eq:vdeq}) is $\mathbb{E}_{q(\mathbf{u} | \mathbf{x})} [ -\log p_X (\mathbf{x} + \mathbf{u}) + \log q(\mathbf{u} | \mathbf{x})]$. It has close relationship between Eq. (\ref{eq:codelength}) in the paper.

\subsection{rANS Coder}

rANS is an efficient entropy coding method. It is the range-based variant of Asymmetric Numeral System (ANS)~\cite{duda2013asymmetric}. 
Given symbol $s$ and the probability mass function $p(s)$, $s$ can be encoded with codelength $-\log_2 p(s)$. 

rANS is very simple to implement as the encoded bits is represented by a single number. For existing code $c$, $s$ can be encoded to $c'$ as
\begin{equation}
    c'(c, s) = \lfloor c / l_s \rfloor \cdot m + (c \mod l_s) + b_s
    \label{eq:rans_encoder}
\end{equation}
In the above equation, $m$ is a large integer and usually be chosen as a power of two. Denote by $\mathrm{cdf}(s)$ as the cumulative density function such that $\mathrm{cdf}(s) = \sum_{i=1}^s p(s)$, we have $b_s = \lfloor \mathrm{cdf}(s-1) \cdot m \rfloor$ and $l_s = b_{s+1} - b_s$.

Given $p(s)$ and code $c'$, the symbol can be recovered. First, we compute $b = c' \mod m$, then $s$ can be recovered such that $b_s \le b < b_{s+1}$. This can be implemented with binary search. Then the code can be recovered with
\begin{equation}
    c(c', s) = \lfloor c' / m \rfloor \cdot l_s + (c' \mod m) - b_s
    \label{eq:rans_decoder}
\end{equation}

We can easily perform bits-back coding with rANS coder. The auxiliary bits can be initialised with certain integer. Then the bits-back decoding can be performed with Eq. (\ref{eq:rans_decoder}) and the code is recovered with Eq. (\ref{eq:rans_encoder}). In this paper, we use rANS for coding with iVPF.

\subsection{From Volume Preserving Flow to iVPF}

With simple modifications, a volume preserving flow can be transformed to an iVPF. 
The model parameter of the iVPF is directly the parameter in the trained continuous volume preserving flow.
Furthermore, the following modifications are performed:

\textbf{(1) Outputs of each layer.} Please refer to Sec. \ref{sec:quantization} and \ref{sec:ivpf}. The inputs and outputs of each flow layer should be $k$-precision quantized. While in continuous flow, quantization is not needed.

\textbf{(2) Coupling layer.} Please refer to Sec. \ref{sec:coupling}. The results should be computed with Alg. \ref{alg:coupling}. While in continuous flow, the outputs can be directly computed with Eq. (\ref{eq:coupling}).

\textbf{(3) $1\times1$ convolution layer.} Please refer to Sec. \ref{sec:conv}. While in continuous flow, the outputs can be directly computed with matrix multiplication of $\mathbf{W}$.


\section{Proofs of Propositions}
\label{app:proof_inv}

In this section, we show the proof of the two propositions in Section 3.2.2.
We represent the two propositions in this section and use Fig. \ref{fig:volume} to illustrate Proposition \ref{cl:inv}.
\begin{figure}[t]
\begin{center}
\includegraphics[width=0.45\textwidth]{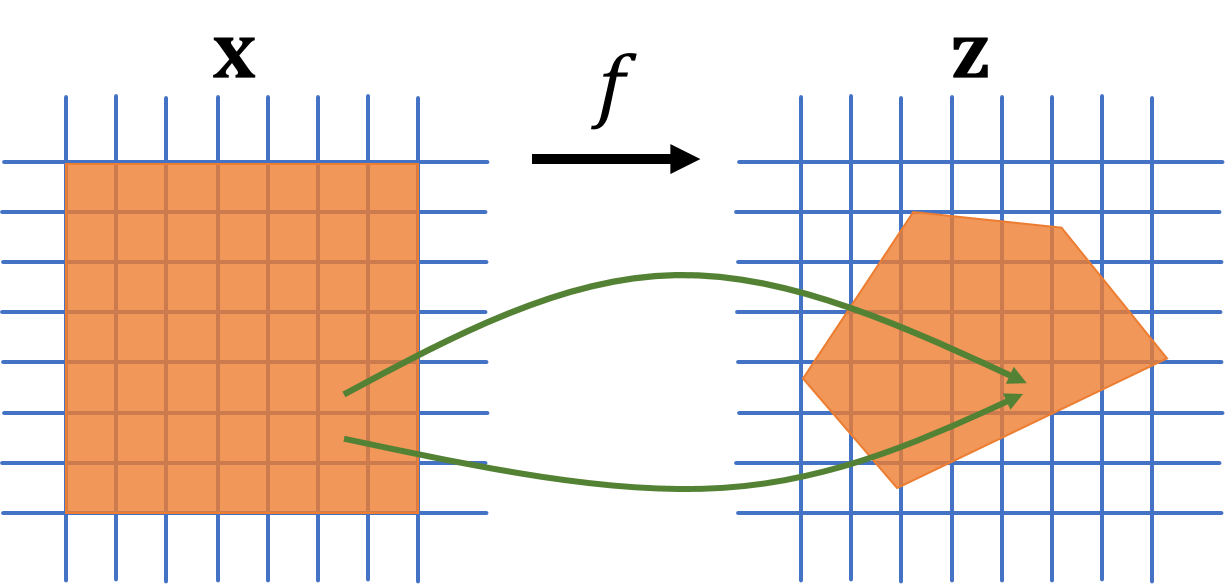}\\
\end{center}
\caption{Illustration of how non-volume preserving flow could lead to no-bijection after discretization (Proposition \ref{cl:inv}). The orange regions denote the domain $\mathcal{U}_0$ (in $\mathbf{x})$ and the corresponding co-domain $\mathcal{V}_0$ (in $\mathbf{z}$).}
\label{fig:volume}
\end{figure}


\textbf{(1) Proof of Proposition \ref{cl:long}.}

\begin{proof}
By change of variable, we have
\begin{equation*}
    p_Z(f(\mathbf{x}_0))|\frac{\partial \mathbf{z}}{\partial \mathbf{x}}(\mathbf{x}_0)| = p_X(\mathbf{x}_0).
\end{equation*}
Since $|\frac{\partial \mathbf{z}}{\partial \mathbf{x}}(\mathbf{x}_0)| > 1$, we have $p_Z(f(\mathbf{x}_0)) < p_X(\mathbf{x}_0)$.
Let $\epsilon=p_X(\mathbf{x}_0)-p_Z(f(\mathbf{x}_0))$. 
Due to the continuity of $p_Z$, there exist $\delta>0$, s.t. $p_Z(\mathbf{z}^\prime)\in (p_Z(\mathbf{z})-\epsilon, p_Z(\mathbf{z})+\epsilon)$, when $||\mathbf{z}-\mathbf{z}^\prime||<\delta$.
$K$ can be chosen st. $2^{-K-1} < \delta$ and we have
\begin{equation*}
    ||f(\mathbf{x}_0) - \round{f(\mathbf{x}_0)}|| \le 2^{-K-1} < \delta,
\end{equation*}
which means $p_Z(\round{f(\mathbf{x}_0)}) < p_X(\mathbf{x}_0)$.

\end{proof}


\textbf{(2) Proof of Proposition \ref{cl:inv}.}

\begin{proof}
Assume $|\frac{\partial \mathbf{z}}{\partial \mathbf{x}}(\mathbf{x}_0)| < 1$, due to the continuity of the flow gradient $f^\prime$, there must be a neighbourhood of $\mathbf{x}_0$ denoted by $\mathcal{U}_{0}$, such that $|\frac{\partial \mathbf{z}}{\partial \mathbf{x}}|<1$ for any $\mathbf{x} \in \mathcal{U}_0$.
Let $\mathcal{V}_0 = \{\mathbf{z}=f(\mathbf{x}): \mathbf{x} \in \mathcal{U}_0 \}$,
we have $|\mathcal{V}_0|<|\mathcal{U}_0|$ where $|\cdot|$ denotes the volume of a space.

Consider a discretisation scheme dividing the space into bins with volume $\delta$ much smaller than $|\mathcal{U}_0|$.
Since the volume of $\mathcal{V}_0$ is strictly smaller than the volume of $\mathcal{U}_0$, the number of bins in $\mathcal{V}_0$ is strictly smaller than that in $\mathcal{U}_0$.
Thus a bijection could not be induced in such scenario for non-volume preserving flows. 
To be specific, there exists multiple $\mathbf{x}$s at different bins that are mapped to the same $\mathbf{z}$, making the lossless compression intractable.
\end{proof}

\section{Correctness of MAT (Alg. \ref{alg:coupling})}
\label{app:proof_mat}

Unless specified, we use the notation in Sec. \ref{sec:coupling}. 

For the encoding process, the input is $(\mathbf{x}_b, r_0)$ where $\round{\mathbf{x}_b} = \mathbf{x}_b, r_0 \in [0, 2^C)$, the parameters are integers $m_0, m_1, ..., m_{d_b}$ where $m_0 = m_{d_b} = 2^C$, and the expected output is $(\mathbf{z}_b, r_{d_b})$. Then $(\mathbf{z}_b, r_{d_b})$ is computed as follows:

\begin{enumerate}
    \item $\bar{\mathbf{x}}_b = 2^k \cdot \mathbf{x}_b$.
    \item Given $\bar{\mathbf{x}}_b = [x_1, ..., x_{d_b}]^\top$ and $r_0$, compute $\mathbf{y}_b = [y_1, ..., y_{d_b}]^\top$ and $v_1, ..., v_{d_b}$ for $i=1,...,d_b$ such that
    \begin{equation*}
    \begin{split}
        &v_i = x_i \cdot m_{i-1} + r_{i-1}; \\
        &y_i = \lfloor v_i / m_i \rfloor, \quad r_i = v_i \mod m_i.
    \end{split}
    \end{equation*}
    Then get $(\mathbf{y}_b, r_{d_b})$.
    \item $\mathbf{z}_b = \mathbf{y}_b / 2^k + \round{\mathbf{t}}$. Then get $(\mathbf{z}_b, r_{d_b})$.
\end{enumerate}

For decoding process, the input is $(\mathbf{z}_b, r_{d_b})$ where $\round{\mathbf{z}_b} = \mathbf{z}_b, r_{d_b} \in [0, 2^C)$, the parameters $m_0, m_1, ..., m_{d_b}$ are the same as that in encoding process with $m_0 = m_{d_b} = 2^C$, and the expected output is $(\mathbf{x}_b, r_{0})$. Then $(\mathbf{x}_b, r_{0})$ is computed as follows:

\begin{enumerate}
    \item $\mathbf{y}_b = 2^k \cdot (\mathbf{z}_b - \round{\mathbf{t}})$.
    \item Compute $\bar{\mathbf{x}}_b = [x_1, ..., x_{d_b}]^\top$ and $v_1 ..., v_{d_b}$ for $i=d_b,...,1$ such that
    \begin{equation*}
    \begin{split}
        &v_i = y_i \cdot m_{i} + r_{i}; \\
        &x_i = \lfloor v_i / m_{i-1} \rfloor, \quad r_{i-1} = v_i \mod m_{i-1}.
    \end{split}
    \end{equation*}
    Then get $(\bar{\mathbf{x}}_b, r_{0})$.
    \item $\mathbf{x}_b = \bar{\mathbf{x}}_b / 2^k$. Then get $(\mathbf{x}_b, r_{0})$.
\end{enumerate}

Then the correctness of MAT in Alg. \ref{alg:coupling} is ensured by the following theorem:
\begin{theorem}
Given the above encoding and decoding process, if $r_0 \in [0, 2^C)$, then

\textbf{P1}: $\round{\mathbf{z}_b} = \mathbf{z}_b, r_{d_b} \in [0, 2^C)$;

\textbf{P2}: $(\mathbf{x}_b, r_{0})$ and $(\mathbf{z}_b, r_{d_b})$ establish a bijection.
\end{theorem}

\begin{proof}
\textbf{P1}. Just consider the encoding process. In Step 1, $\mathbf{x}_b$ is $k$-precision floating points, $\bar{\mathbf{x}}_b$ are integers. In Step 2, $\mathbf{y}_b$ are integers as all elements are computed with integer space. As $r_{d_b} = v_{d_b} \mod m_{d_b}$ and $m_{d_b} = 2^C$, $r_{d_b} \in [0, 2^C)$. In Step 3, both $\mathbf{y}_b / 2^k$ and $\round{\mathbf{t}}$ are $k$-precision floating points, thus $\mathbf{z}_b$ is $k$-precision ones with $\round{\mathbf{z}_b} = \mathbf{z}_b$. Then the proof of \textbf{P1} is completed.

\textbf{P2}. Consider Step 3 in the encoding process and Step 1 in the decoding process, as $\round{\mathbf{z}_b} = \mathbf{z}_b$, $\mathbf{z}_b$ and $\mathbf{y}_b$ establish a bijection. Consider Step 1 in the encoding process and Step 3 in the decoding process, as $\round{\mathbf{x}_b} = \mathbf{x}_b$, $\mathbf{x}_b$ and $\bar{\mathbf{x}}_b$ establish a bijection. Then \textbf{P2} induces the bijection between $(\bar{\mathbf{x}}_b, r_{0})$ and $(\mathbf{y}_b, r_{d_b})$.

Then we prove $(\bar{\mathbf{x}}_b, r_{0})$ and $(\mathbf{y}_b, r_{d_b})$ establish a bijection. For ease of analysis, we rewrite the Step 2 in the decoding process such that $v_i^d = y_i \cdot m_{i} + r_{i}; x_i^d = \lfloor v_i^d / m_{i-1} \rfloor, r_{i-1}^d = v_i^d \mod m_{i-1}$. We should prove $x_i^d = x_i$ for all $i=1,...,d_b$ and $r_0^d = r_0$, which is done with mathematical induction.

(i) For $l = d_b$, we prove that $(x_l^d, r_{l-1}^d) = (x_l, r_{l-1})$. In fact, For the encoding process, $x_l \cdot m_{l-1} + r_{l-1} = v_l^e =  y_l \cdot m_l + r_l$ and $r_l \in [0, m_l) (m_l = m_{d_b} = 2^C)$. For the decoding process, we have $v_l^d = y_l \cdot m_l + r_l = v_l = x_l \cdot m_{l-1} + r_{l-1}$. As $r_{l-1} \in [0, m_{l-1})$, it is clear that $x_l^d = \lfloor (x_l \cdot m_{l-1} + r_{l-1}) / m_{l-1} \rfloor = x_l, r_{l-1}^d = (x_l \cdot m_{l-1} + r_{l-1}) \mod m_{l-1} = r_{l-1}$.

(ii) For $1 \le l < d_b$, if $(x_{l+1}^d, r_{l}^d) = (x_{l+1}, r_{l})$, we prove that $(x_l^d, r_{l-1}^d) = (x_l, r_{l-1})$. For the encoding process, $x_l \cdot m_{l-1} + r_{l-1} = v_l =  y_l \cdot m_l + r_l$. For the decoding process, $v_l^d = y_l \cdot m_l + r_l^d$. As $r_l^d = r_l$, we have $v_l^d = v_l$. As $r_{l-1} \in [0, m_{l-1})$, it is clear that $x_l^d = \lfloor (x_l \cdot m_{l-1} + r_{l-1}) / m_{l-1} \rfloor = x_l, r_{l-1}^d = (x_l \cdot m_{l-1} + r_{l-1}) \mod m_{l-1} = r_{l-1}$.

Note that we use a fact that $r_l \in [0, m_l), l = 1,...,d_b$ in (i)(ii), which can be directly derived from $r_l = v_l \mod m_l$.

From (i)(ii), we have $x_l^d = x_l$ for all $l=1,...,d_b$ and $r_0^d = r_0$. It conveys that $(\mathbf{x}_b, r_{0})$ and $(\mathbf{z}_b, r_{d_b})$ are bijections, which completes the proof of \textbf{P2}.
\end{proof}

\section{Detailed Numerical Error Analysis of iVPF}
\label{app:error}

In this section, we perform detailed numerical error analysis for iVPF. As the iVPF model mainly contains coupling layer and $1 \times 1$ convolution layer, we focus on error analysis on these layers. Overall error analysis is performed based on the composition of flow layers.

In the rest of this section, the notation is the same as that in the main paper.
We first emphasise that the error between $x$ and the quantized $x$ at precision $k$, denoted by $\round{x}$, is
\begin{equation*}
    |x - \round{x}| \le 2^{-k-1}.
\end{equation*}

\begin{table*}[t]
\centering
\small
\caption{Coding efficiency of LBB and iVPF on CIFAR10. $k=14$ is used in iVPF.}
\label{tab:timing_lbb}
\begin{tabular}{lcccc}
\toprule
         & batch size & inference time (ms) & time w/ coding (ms) & \# coding \\
\midrule
LBB~\cite{ho2019compression} & 64 & 16.2 & 116 & 188 \\
 & 256 & 13.8 & 97 & 188 \\
\midrule
\textbf{iVPF (Ours)} & 64 & \textbf{10.9} & \textbf{11.4} & \textbf{2} \\
 & 256 & \textbf{6.1} & \textbf{6.6} & \textbf{2} \\
\bottomrule
\end{tabular}
\end{table*}

\subsection{MAT and Coupling Layer}
\label{sec:err_coupling}

For all $i = 1,2,...,{d_b}-1$, $m_i = \mathrm{round}(m_0 / \prod_{j=1}^i s_j)$; and $m_{d_b} = m_0 = \mathrm{round}(m_0 / \prod_{j=1}^i s_j)$ as $\prod_{j=1}^{d_b} s_j = 1$. Then we have $m_{i} - 0.5 \le m_0 / \prod_{j=1}^i s_j \le m_{i} + 0.5$, and $m_{i-1} - 0.5 \le m_0 / \prod_{j=1}^{i-1} s_j \le m_{i-1} + 0.5$. Thus $s_i$ is approximated with $m_{i-1} / m_i$ such that
\begin{equation*}
    \frac{m_{i-1} - 0.5}{m_i + 0.5} \le s_i \le \frac{m_{i-1} + 0.5}{m_i - 0.5}
\end{equation*}
Considering $s_i$ is fixed and $m_0 = m_{d_b} = 2^C$, it is clear that $m_i = O(2^C)$. As $\frac{m_{i-1}}{m_i} - \frac{m_{i-1} - 0.5}{m_i + 0.5} = \frac{1}{m_i} \cdot \frac{0.5(m_{i-1} + m_i)}{m_i + 0.5} = O(2^{-C})$, therefore we have $s \ge \frac{m_{i-1}}{m_i} - O(2^C)$. Similar conclusion can be arrived such that $s \le \frac{m_{i-1}}{m_i} + O(2^C)$ and 
\begin{equation*}
    |s_i - \frac{m_{i-1}}{m_i}| \le O(2^{-C}).
\end{equation*}

Then we investigate the error between $y_i = \lfloor (2^k \cdot x_i \cdot m_{i-1} + r) / m_i \rfloor / 2^k \ (0 \le r < m_{i-1})$ and $s_i \cdot x_i$. Note that $2^k$ in the above equality correspond to Line 2 and 12 in Alg.~\ref{alg:coupling}. In fact, $y_i \le (2^k \cdot x_i \cdot m_{i-1} + r) / (2^k \cdot m_i) < (2^k \cdot x_i \cdot m_{i-1} + m_{i-1}) / (2^k \cdot m_i) = x_i \cdot m_{i-1} / m_i + m_{i-1} / (2^k \cdot m_i) =  x_i \cdot m_{i-1} / m_i + O(2^{-k})$, and $y_i > [(2^k \cdot x_i \cdot m_{i-1} + r) / m_i - 1] / 2^k >  x_i \cdot m_{i-1} / m_i - 1 / 2^k = x_i \cdot m_{i-1} / m_i - O(2^{-k})$. With $|s_i - \frac{m_{i-1}}{m_i}| \le O(2^{-C})$, the error is estimated such that
\begin{equation*}
\begin{split}
    |y_i - s_i \cdot x_i| &\le |y_i - \frac{m_{i-1}}{m_i} x_i| + O(2^{-C}) \\
    &\le O(2^{-k}) + O(2^{-C}) = O(2^{-k}, 2^{-C})
\end{split}
\end{equation*}

As $z_i = y_i + \round{t_i}$ and $|\round{t_i} - t_i| = O(2^{-k})$, we finally have
\begin{equation*}
\begin{split}
    |z_i - (s_i x_i + t_i)| &\le |y_i - s_i \cdot x_i| + |\round{t_i} - t_i| \\ 
    &= O(2^{-k}, 2^{-C}) + O(2^{-k}) = O(2^{-k}, 2^{-C})
\end{split}
\end{equation*}
which is consistent with Eq. (\ref{eq:coupling_error}) in the main paper. Thus the coupling layer in iVPF brings $O(2^{-k}, 2^{-C})$ error.

\subsection{$1 \times 1$ Convolution Layer}
\label{sec:err_conv}

For any $\mathbf{x}$, denote by $\Tilde{z}_i$ the $i$th element of $\mathbf{U} \mathbf{x}$ and by $z_i$ that derived in Eq. (\ref{eq:upper}), it is clear that $|z_i - \Tilde{z}_i| = O(2^{-k})$ as only quantization operation is involved.

For any $\mathbf{x}$, denote by $\Tilde{z}_i$ the $i$th element of $\mathbf{\Lambda} \mathbf{x}$ and by $z_i$ that derived with Alg. \ref{alg:coupling}, according to Sec. \ref{sec:err_coupling}, we have $|z_i - \Tilde{z}_i| = O(2^{-k}, 2^{-C})$.

For any $\mathbf{x}$, denote by $\Tilde{z}_i$ the $i$th element of $\mathbf{L} \mathbf{x}$ and by $z_i$ that derived in Eq. (\ref{eq:lower}), it is clear that $|z_i - \Tilde{z}_i| = O(2^{-k})$ as only quantization operation is involved.

No numerical error is involved by performing permutation operation with the permutation matrix $\mathbf{P}$.

As the convolution layer involves sequential matrix multiplications of $\mathbf{U}, \mathbf{\Lambda}, \mathbf{L}$ and $\mathbf{P}$, the error is accumulated. 
Denote by $\mathbf{z} = \mathbf{P} \mathbf{L} \mathbf{\Lambda} \mathbf{U}$ and $\Bar{\mathbf{z}}$ the output of the corresponding iVPF layer.
We have
\begin{equation*}
    |\Bar{\mathbf{z}} - \mathbf{z}| = O(2^{-k}, 2^{-C}).
\end{equation*}

\subsection{Overall Error Analysis}

Denote by $\mathbf{z}$ the output of the continuous volume preserving flow model and by $\bar{\mathbf{z}}$ the output of iVPF model, then we will show that 
\begin{equation*}
    |\bar{\mathbf{z}} - \mathbf{z}| = O(L2^{-k}, L2^{-C}).
\end{equation*}

Consider the continuous flow layer $f_l$ and the corresponding iVPF layer $\bar{f}_l$.
With the same input $\mathbf{y}_{l-1}$, we have $|\bar{f}_l(\mathbf{y}_{l-1}) - f_l(\mathbf{y}_{l-1})| = O(2^{-k}, 2^{-C})$.
It usually holds that $f$ is $\lambda_l$-Lipschitz continuous, and therefore $|f_l(\bar{\mathbf{y}}_{l-1}) - f_l(\mathbf{y}_{l-1})| \le \lambda_l |\bar{\mathbf{y}}_l - \mathbf{y}_l|$. 
Thus for the volume preserving flow model with $\mathbf{y}_0 = \mathbf{x}, \mathbf{z} = \mathbf{y}_L, \mathbf{y}_l = f_l (\mathbf{y}_{l-1})$, and the corresponding iVPF model with $\bar{\mathbf{y}}_0 = \mathbf{x}, \bar{\mathbf{z}} = \bar{\mathbf{y}}_L, \bar{\mathbf{y}}_l = \bar{f}_l (\bar{\mathbf{y}}_{l-1})$, we have 
\begin{equation*}
\begin{split}
    |\bar{\mathbf{y}}_l - \mathbf{y}_l| &= |\bar{f}_l(\bar{\mathbf{y}}_{l-1}) - f_l(\mathbf{y}_{l-1})| \\
    &\le |\bar{f}_l(\bar{\mathbf{y}}_{l-1}) - f_l(\bar{\mathbf{y}}_{l-1})| + |f_l(\bar{\mathbf{y}}_{l-1})  - f_l(\mathbf{y}_{l-1})| \\
    &\le  O(2^{-k}, 2^{-C}) + \lambda_l |\bar{\mathbf{y}}_{l-1} - \mathbf{y}_{l-1}|
\end{split}
\end{equation*}
and therefore
\begin{equation*}
\begin{split}
    |\bar{\mathbf{z}} - \mathbf{z}| &= \sum_{l=1}^L (\prod_{k=l+1}^L \lambda_k) O(2^{-k}, 2^{-C}) \\
    &= O(L2^{-k}, L2^{-C})
\end{split}
\end{equation*}
the last equality hold as $\prod_{k=l+1}^L \lambda_k$ is bounded, which is guaranteed by the volume-preserving property.

From the above formula, if $L$ is limited and $2^{-k}$ and $2^{-C}$ are relatively small, the latents generated with the iVPF model is able to approximate the true distribution.

\section{More Experiments on Coding Efficiency}
\label{app:label}

The experiment is conducted on PyTorch framework with Tesla P100 GPU and Intel(R) Xeon(R) CPU E5-2690 v4 @ 2.6GHz GPU. We adapt the ANS decoding and encoding code from LBB~\cite{ho2019compression}, which is very time-efficient. The coding time is shown in Table \ref{tab:timing_lbb}. It is clear that the proposed iVPF method is much faster than LBB. Moreover, as LBB involves a large number of coding schemes, ANS coding takes most of the time, while model inference time takes the main time in iVPF.

\end{document}